\documentclass{article}

\usepackage[final, nonatbib]{nips_2018}

\usepackage[utf8]{inputenc} %
\usepackage[T1]{fontenc}    %
\usepackage{hyperref}       %
\usepackage{url}            %
\usepackage{booktabs}       %
\usepackage{amsfonts}       %
\usepackage{nicefrac}       %
\usepackage{microtype}      %

 \usepackage[textsize=tiny,disable]{todonotes}
\usepackage[textsize=tiny]{todonotes}

\usepackage{amsmath}
\usepackage{amsfonts}
\usepackage{amssymb}
\usepackage{mathtools}
\usepackage{graphicx}
\usepackage{amsthm}
\usepackage{algorithmic}
\usepackage[linesnumbered,ruled]{algorithm2e}

\newcommand{\cN}{\mathcal{N}}

\newcommand{\cH}{\mathcal{H}}

\newcommand{\cF}{\mathcal{F}}

\newcommand{\cK}{\mathcal{K}}

\newcommand{\R}{\mathbb{R}}

\newcommand{\Ev}{\mathbb{E}}

\newcommand{\cL}{\mathcal{L}}

\newcommand{\de}{\partial}

\newcommand{\bs}{\boldsymbol}

\newcommand{\mpr}{m_{\textnormal{pr}}}
\newcommand{\mpos}{m_{\textnormal{pos}}}
\newcommand{\Cpos}{C_{\textnormal{pos}}}

\newcommand{\Cpr}{C_{\textnormal{pr}}}

\newcommand{\KL}{\mathcal{D}_\textnormal{KL}}

\newcommand{\trace}{{\textnormal{trace}}}

\newcommand{\xtrue}{x_{\textnormal{true}}}

\newtheorem{theorem}{Theorem}
\newtheorem{corollary}{Corollary}
\newtheorem{proposition}{Proposition}

\title{A Stein variational Newton method}

\author{
Gianluca Detommaso\\
University of Bath \& The Alan Turing Institute\\
\texttt{gd391@bath.ac.uk} \\
\And
Tiangang Cui\\
Monash University\\
\texttt{Tiangang.Cui@monash.edu} \\
\And
Alessio Spantini\\
Massachusetts Institute of Technology\\
\texttt{spantini@mit.edu}
\And
Youssef Marzouk\\
Massachusetts Institute of Technology\\
\texttt{ymarz@mit.edu} \\
\And
Robert Scheichl\\
Heidelberg University\\
\texttt{r.scheichl@uni-heidelberg.de} 
}

\begin{document}

\maketitle

\begin{abstract}

  Stein variational gradient descent (SVGD) was recently proposed as a
  general purpose nonparametric variational inference algorithm [Liu
  \& Wang, NIPS 2016]: it minimizes the Kullback--Leibler divergence
  between the target distribution and its approximation by
  implementing a form of functional gradient descent on a reproducing
  kernel Hilbert space. In this paper, we accelerate and generalize
  the SVGD algorithm by including second-order information,
  thereby approximating a Newton-like iteration in function space. We
  also show how second-order information can lead to more effective
  choices of kernel.  We observe significant computational gains over
  the original SVGD algorithm in multiple test cases.

\end{abstract}

\section{Introduction}

Approximating an intractable probability distribution via a collection
of samples---in order to evaluate arbitrary expectations over the
distribution, or to otherwise characterize uncertainty that the
distribution encodes---is a core computational challenge in statistics
and machine learning.
Common features of the target distribution can make sampling a
daunting task. For instance, in a typical Bayesian inference problem,
the posterior distribution might be strongly non-Gaussian (perhaps
multimodal) and high dimensional, and evaluations of its density might
be computationally intensive.

There exist a wide range of algorithms for such
problems, ranging from parametric variational inference
\cite{blei2017variational} to Markov chain Monte Carlo (MCMC)
techniques \cite{gilks1995markov}.  Each algorithm offers a different
computational trade-off.  At one end of the spectrum, we find the
parametric mean-field approximation---a cheap but potentially
inaccurate variational approximation of the target density.  
At the other end, we find MCMC---a nonparametric sampling technique
yielding estimators that are consistent, but potentially slow to
converge.  In this paper, we focus on Stein variational gradient
descent (SVGD) \cite{liu2016svgd}, which lies somewhere in
the middle of the spectrum and can be described as a
particular {\it nonparametric} variational inference method
\cite{blei2017variational}, with close links to the density estimation
approach in \cite{AnderesCoram2012}. 

The SVGD algorithm seeks a deterministic coupling between a tractable
reference distribution of choice (e.g., a standard normal) and the
intractable target.  This coupling is induced by a transport map $T$
that can {\it transform} a collection of reference samples into
samples from the desired target distribution.  For a given pair of
distributions, there may exist infinitely many such maps
\cite{villani2008optimal}; several existing algorithms (e.g.,
\cite{tabak2013family,rezende2015variational,marzouk2016sampling}) aim
to approximate feasible transport maps of various forms.
The distinguishing feature of the SVGD algorithm
lies in its definition of a suitable map $T$.
Its central idea is to approximate $T$ as a growing composition of
simple maps, computed \emph{sequentially}:
\begin{equation}
T = T_1 \circ \cdots \circ T_k \circ \cdots,
\end{equation}
where each map $T_k$ is a perturbation of the identity map along the
steepest descent direction of a functional $J$ that describes the
Kullback--Leibler (KL) divergence between the pushforward of the
reference distribution through the composition
$T_1\circ\cdots\circ T_{k}$ and the target distribution.
The steepest descent direction is further {\it projected} onto a
reproducing kernel Hilbert space (RKHS) in order to give $T_k$ a
nonparametric closed form \cite{aronszajn1950theory}.
Even though the resulting map $T_k$ is available explicitly without
any need for numerical optimization, the SVGD algorithm implicitly
approximates a steepest descent iteration on a space of maps of given
regularity.

A primary goal of this paper is to explore the use of second-order
information (e.g., Hessians) within the SVGD algorithm.  Our idea is
to develop the analogue of a Newton iteration---rather than gradient
descent---for the purpose of sampling distributions more efficiently.
Specifically, we design an algorithm where each map $T_k$ is now
computed as the perturbation of the identity function along the
direction that minimizes a certain local quadratic approximation of $J$.
Accounting for second-order information can dramatically accelerate
convergence to the target distribution, at the price of additional
work per iteration.  The tradeoff between speed of convergence and
cost per iteration is resolved in favor of the Newton-like
algorithm---which we call a Stein variational Newton method (SVN)---in
several numerical examples.

The efficiency of the SVGD and SVN algorithms depends further on the
choice of reproducing kernel. A second contribution of this paper is
to design geometry-aware Gaussian kernels that also exploit
second-order information, yielding substantially faster convergence
towards the target distribution than SVGD or SVN with an isotropic
kernel.

In the context of {\it parametric} variational inference, second-order
information has been used  to accelerate the convergence of certain variational approximations, e.g., \cite{khan2017adaptivenewton,khal2017viRMSprop,marzouk2016sampling}.
In this paper, however, we focus on {\it nonparametric} variational
approximations, where the corresponding optimisation problem is
defined over an infinite-dimensional RKHS of transport maps.  More
closely related to our work is the Riemannian SVGD algorithm
\cite{liu2017riemannian}, which generalizes a gradient flow
interpretation of SVGD \cite{liu2017stein} to Riemannian manifolds,
and thus also exploits geometric information within the inference
task.

The rest of the paper is organized as follows.  Section \ref{sec:SVGD}
briefly reviews the SVGD algorithm, and Section \ref{sec:secordinfo}
introduces the new SVN method.  In Section \ref{sec:kernel} we
introduce geometry-aware kernels for the SVN method. Numerical
experiments are described in Section \ref{sec:testcases}.
Proofs of our main results and further numerical examples
addressing scaling to high dimensions are
given in the Appendix.
Code and all numerical examples are collected in our GitHub repository \cite{github}.

%
%
%
%
%
%
%
%
%
%

%

%
%
%
%
%
%
%
%
%
%
%

%
%
%
%
%
%
%
%
%
%
%
%
%

%
%
%

%

%

%
%

%

\section{Background}\label{sec:SVGD}

Suppose we wish to 
approximate
an intractable target %
distribution with density $\pi$ on
$\R^{d}$ via an empirical measure, i.e., a collection of samples. 
Given samples $\{x_i\}$ from a tractable reference density $p$ on $\R^{d}$, one can seek 
a transport map $T:\R^{d}\to\R^{d}$ such that the pushforward 
density of $p$ under $T$, denoted by $T_\ast p$, is a close approximation to the target $\pi$.\footnote{
If $T$ is invertible, then %
$T_*p(x)=p(T^{-1}(x))\,|\det( \nabla_x T^{-1}(x) )|$.
}
There exist infinitely many such maps \cite{villani2008optimal}.
The image of the reference samples under the map, $\{T(x_i)\}$, can then
serve as an empirical measure approximation of $\pi$ 
(e.g., in the weak sense \cite{liu2016svgd}).
\paragraph{Variational approximation.} 
Using the KL divergence to measure the discrepancy 
between the target $\pi$ and the pushforward $T_\ast p$, one can look for a transport map $T$ that minimises the 
functional
\begin{equation} \label{Jp}
T \mapsto \KL(T_*\,p\,||\,\pi)
\end{equation}
over a broad class of functions.
The Stein variational method 
breaks the minimization of \eqref{Jp} into 
several simple steps: it
builds
a sequence of transport maps $\{T_1, T_2, \ldots, T_l, \ldots\}$ to iteratively push an initial reference density $p_0$ towards $\pi$.
Given a scalar-valued RKHS $\cH$ with a
positive definite kernel $k(x, x^\prime)$, 
each transport map $T_l:\R^{d}\to\R^{d}$ is 
chosen to be    
a perturbation of the identity map $I(x)=x$ along the
vector-valued RKHS 
$\cH^d \simeq \cH \times \cdots \times \cH$, i.e.,
\begin{equation}\label{eq:map}
T_l(x) \coloneqq I(x) + Q(x)  \textrm{\quad for \quad }  Q\in\cH^d.%
\end{equation}
The transport maps 
are
computed iteratively.
At each iteration $l$, 
our best approximation of $\pi$
is given by the pushforward density 
$p_l = (T_l \circ \cdots \circ T_1)_\ast \,p_0$.
The SVGD algorithm then seeks
a transport map
$T_{l+1} = I + Q$
that further
decreases 
the KL divergence between $(T_{l+1})_\ast p_l$ and $\pi$,
\begin{equation}  \label{eq:objective_KL}
Q \mapsto J_{p_l}[Q] \coloneqq \KL((I+Q)_*\,p_l\,||\,\pi),
\end{equation}
for an appropriate choice of $Q\in \cH^d$.
In other words, the SVGD algorithm seeks a map $Q \in \cH^d$ such that
\begin{equation} \label{eq:monot}
J_{p_l}[Q] < J_{p_l}[{\bf 0}],
\end{equation}
where ${\bf 0}(x)=0$ denotes the zero map.
By construction, the sequence of pushforward
densities $\{p_0, p_1, p_2, \ldots, p_l, \ldots \}$ becomes
increasingly closer (in KL divergence) to the target $\pi$.
Recent results on the convergence 
of the SVGD algorithm are presented in
\cite{liu2017stein}.

\paragraph{Functional gradient descent.} 
The first variation of $J_{p_l}$ at $S \in \cH^d$ along $V \in \cH^d$ can be defined as
\begin{equation}
D J_{p_l}[S] (V) \coloneqq \lim_{\tau \rightarrow 0} \frac{1}{\tau}\big( J_{p_l}[S + \tau V] - J_{p_l}[S] \big).
\end{equation}
Assuming that the objective function $J_{p_l}: \cH^d \to \R$ is Fr\'{e}chet differentiable, the \emph{functional gradient} of $J_{p_l}$ at $S \in \cH^d$ is the
element $\nabla J_{p_l}[S]$ of $\cH^d$ such that  %
\begin{equation} \label{eq:grad}
D J_{p_l}[S] (V) = \langle \nabla J_{p_l}[S], V \rangle_{\cH^d} \quad \forall \, V \in \cH^d,
\end{equation}
where $\langle \cdot , \cdot  \rangle_{\cH^d}$
denotes an inner product on $\cH^d$.

In order to satisfy \eqref{eq:monot},
the SVGD algorithm %
defines $T_{l+1}$ as a perturbation of the identity map
along the steepest descent direction of the
functional $J_{p_l}$ evaluated at the zero map, i.e.,
\begin{equation} \label{eq:step_desc}
T_{l+1} = I -\varepsilon \nabla J_{p_l}[{\bf 0}],
\end{equation}
for a small enough $\varepsilon > 0$.
It was shown in \cite{liu2016svgd}  that the functional gradient at
${\bf 0}$ has
a closed form expression given by 
\begin{equation}\label{eq:grad_l}
- \nabla J_{p_l}[{\bf 0}](z) =  \Ev_{x\sim p_l}[k(x,z) \nabla_x\log\pi(x) + \nabla_xk(x, z)].
\end{equation}

\paragraph{Empirical approximation.} 
There are several ways to approximate the expectation in \eqref{eq:grad_l}.
For instance, a set of particles $\{x_i^0\}_{i = 1}^{n}$ can be generated from the initial reference density $p_0$ and pushed forward by the transport maps
$\{T_1, T_2, \ldots \}$.
The pushforward density  $p_l$ can then be approximated by the
empirical measure given by the particles $\{x_i^l\}_{i = 1}^{n}$, where $x_i^l = T_l(x_i^{l-1})$ for $i = 1, \ldots, n$, so that
\begin{equation}\label{eq:MCgradJ}
- \nabla J_{p_l}[{\bf 0}](z) \approx G(z) \coloneqq \frac{1}{n}\textstyle\sum_{j=1}^n\big[k(x_j^l,z) \nabla_{x_j^l}\log\pi(x_j^l) + \nabla_{x_j^l}k(x_j^l,z)\big]\,.
\end{equation}
The first term in 
\eqref{eq:MCgradJ} corresponds to a weighted average steepest descent direction of the log-target density $\pi$ with respect to $p_l$. This
term is responsible for transporting  particles towards high-probability regions of 
$\pi$. 
In contrast, the second term can be viewed as a ``repulsion force'' that spreads the particles along the support of $\pi$,
preventing
them from collapsing around the mode of $\pi$. 
The SVGD algorithm is summarised in Algorithm 1. 

\begin{algorithm}[H]
\caption{One iteration of the Stein variational gradient algorithm}\label{alg_stein}
\SetKwInOut{Input}{Input}
\SetKwInOut{Output}{Output}
\Input{Particles $\{x_i^l\}_{i=1}^n$ at previous iteration $l$; step size $\varepsilon_{l+1}$}
\Output{Particles $\{x_i^{l+1}\}_{i=1}^n$ at new iteration $l+1$}
\begin{algorithmic}[1]
\FOR{$i=1,2,\ldots, n$}
\STATE Set $x_i^{l+1} \gets x_i^l + \varepsilon_{l+1}\, G(x_i^{l})$, where $G$ is defined in \eqref{eq:MCgradJ}.
\ENDFOR
\end{algorithmic}
\end{algorithm}

\section{Stein variational Newton method}\label{sec:secordinfo}

Here we propose a new method that incorporates second-order information to accelerate the convergence of the SVGD algorithm. %
We replace the steepest descent direction in
\eqref{eq:step_desc} with an approximation of the 
Newton direction.

\paragraph{Functional Newton direction.} 
Given a differentiable objective function $J_{p_l}$,
we can define the second variation of $J_{p_l}$ at ${\bf 0}$ along the pair of directions $V, W \in \cH^d$  as %
\[
D^2 J_{p_l}[{\bf 0}] (V, W) \coloneqq 
\lim_{\tau \rightarrow 0} \frac{1}{\tau}\big( DJ_{p_l}[\tau W](V) - DJ_{p_l}[{\bf 0}](V) \big).
\]
At each iteration, the  Newton method seeks
to minimize a local quadratic approximation of $J_{p_l}$.
The minimizer $W \in \cH^d$ of this quadratic form defines the Newton direction and is 
characterized by 
the first order 
stationarity 
conditions
\begin{equation} \label{eq:bil_form}
D^2 J_{p_l}[{\bf 0}] (V, W)  = - D J_{p_l}[{\bf 0}] (V), \quad \forall \, V \in \cH^d.
\end{equation}
We can then look for a transport map $T_{l+1}$ that is 
a local perturbation of the identity map along
the Newton direction, i.e.,
\begin{equation} \label{eq:newton_map}
T_{l+1} = I + \varepsilon W,
\end{equation}
for some $\varepsilon > 0$ that
satisfies \eqref{eq:monot}.
The function $W$ is guaranteed to be a descent direction
if the bilinear form $D^2 J_{p_l}[{\bf 0}]$ in 
\eqref{eq:bil_form}
is positive definite.
The following theorem gives an explicit
form for $D^2 J_{p_l}[{\bf 0}]$ and is
proven in Appendix.

\begin{theorem}\label{thm:hessian}
The variational characterization
of the Newton direction
$W=(w_1,\dots,w_d)^\top\in\cH^d$ in 
\eqref{eq:bil_form} is equivalent to
\begin{equation} \label{eq:variationalForm}
\sum_{i=1}^d\left\langle
\sum_{j=1}^d
\left\langle  h_{ij}(y,z), w_j(z)\right\rangle_{\cH} + \partial_i J_{p_l}[\bs 0](y), v_i(y)
\right\rangle_{\cH} = 0, 
\end{equation}
for all $V=(v_1,\dots,v_d)^\top\in\cH^d$, where
\begin{equation} \label{eq:sym_form}
h_{ij}(y,z) = \mathbb{E}_{x\sim p_l}\left[ \,-\partial_{ij}^2\log \pi(x) k(x,y) k(x,z) + \partial_i k(x, y)\partial_j k(x, z)\,\right].
\end{equation}
\end{theorem}
We 
propose
a Galerkin
approximation of \eqref{eq:variationalForm}.
Let $(x_k)_{k=1}^n$ be an ensemble of particles
distributed according to $p_l(\,\cdot\,)$, and define the finite dimensional linear space 
$\cH^d_n = {\rm span}\{ k(x_1, \cdot), \ldots, k(x_n, \cdot) \}$.
We look for 
an approximate
solution 
$W = (w_1,\ldots,w_d)^\top$ in $\cH^d_n$---i.e.,
\begin{equation} \label{eq:wi}
  w_j(z) = \sum_{k=1}^n \alpha^k_j \,k(x_k, z)
\end{equation}
for some unknown coefficients $(\alpha_j^k)$---such
that the residual of
\eqref{eq:variationalForm} is orthogonal to 
$\cH^d_n$.
The following corollary gives an 
explicit characterization of the Galerkin solution 
and is proven in the Appendix.
\begin{corollary}\label{coro:galerkin}
The 
coefficients $(\alpha_j^k)$ 
are given by the solution of the 
linear
system
\begin{equation}  \label{eq:semi-comp}
 \sum_{k=1}^n \,  
 H^{s,k}\,\alpha^k   = \nabla J^s, \textrm{\quad for\;all\quad} s=1,\ldots,n,
\end{equation}
where $\alpha^k := \big(\alpha^k_1, \ldots, \alpha^k_d\big)^\top$ is a vector of unknown coefficients, 
$(H^{s, k})_{ij} := h_{ij}(x_s,x_k)$ is 
the evaluation of the symmetric form 
\eqref{eq:sym_form} at pairs of particles,
and where
$\nabla J^s := - \nabla J_{p_l}[\bs 0](x_s)$
represents the evaluation of the first variation at the
$s$-th particle.
\end{corollary}
In practice, we can only evaluate a 
Monte Carlo
approximation of
$H^{s, k}$  and $\nabla J^s$ in 
\eqref{eq:semi-comp} using the
ensemble
$(x_k)_{k=1}^n$.
\paragraph{Inexact Newton.} 
The solution of  \eqref{eq:semi-comp} by means
of direct solvers might be
impractical for problems with a large number of particles $n$ or high
parameter dimension $d$, since
it is a linear system  with $n d$ unknowns.
Moreover, the solution of \eqref{eq:semi-comp} might not lead to a descent direction (e.g., when $\pi$ is not log-concave).
We address these issues by deploying two well-established techniques in nonlinear optimisation \cite{wright1999numerical}.
In the first approach, we solve
\eqref{eq:semi-comp} using
the inexact %
Newton--conjugate gradient (NCG) method \cite[Chapters 5 and 7]{wright1999numerical},
wherein
 a descent direction can be guaranteed by appropriately terminating the conjugate gradient 
 iteration.
The NCG method only needs to evaluate the matrix-vector product with
each $H^{s, k}$ and does not construct the matrix explicitly, and thus
can be scaled to high dimensions.
In the second approach, we simplify the problem further by taking a block-diagonal approximation of the second variation,
breaking 
\eqref{eq:semi-comp} into $n$ decoupled $d\times d$ linear systems 
\begin{equation} 
H^{s,s}\alpha^s = \nabla J^s, \qquad s=1,\dots n\,. 
\label{eq:decoupled}
\end{equation}
Here, we can 
either employ a Gauss-Newton approximation of the Hessian $\nabla^2
\log\pi$ in $H^{s,s}$ or again use inexact Newton--CG, to guarantee that the 
 approximation of the Newton direction is a descent direction.

Both the block-diagonal approximation and inexact NCG are more efficient than solving for the full Newton direction \eqref{eq:semi-comp}. In addition, the block-diagonal form \eqref{eq:decoupled} can be solved in parallel for each of the blocks, and hence it may best suit high-dimensional applications and/or large numbers of particles.
In the Appendix, we provide a comparison of these approaches on various examples. Both approaches provide similar progress per SVN iteration compared to the full Newton direction. 

Leveraging second-order information provides a
natural scaling for the step size, i.e., $\varepsilon  = O(1)$.
Here, the choice $\varepsilon  = 1$ performs reasonably well in our numerical experiments (Section \ref{sec:testcases} and the Appendix). In future work, 
we will refine our strategy by considering
either a line search or a trust region step.
The resulting Stein variational Newton method is summarised in Algorithm 2.

\begin{algorithm}[H]
\caption{One iteration of the Stein variational Newton algorithm}\label{alg_stein_newton}
\SetKwInOut{Input}{Input}
\SetKwInOut{Output}{Output}
\Input{Particles $\{x_i^{l}\}_{i=1}^n$ at stage $l$; step size $\varepsilon$}
\Output{Particles $\{x_i^{l+1}\}_{i=1}^n$ at stage $l+1$}
\begin{algorithmic}[1]
\FOR{$i=1,2,\ldots,n$}
\STATE Solve the linear system \eqref{eq:semi-comp} for $\alpha^{1}, \ldots, \alpha^{n}$ 
\STATE Set $x_i^{l+1} \gets x_i^{l} + \varepsilon W(x_i^{l})$ given $\alpha^{1}, \ldots, \alpha^{n}$ 
\ENDFOR
\end{algorithmic}
\end{algorithm}

\section{Scaled Hessian kernel}\label{sec:kernel}
In the Stein variational method, the kernel weighs the contribution of each particle to a locally \textit{averaged} steepest descent direction of the target distribution, and it also spreads the particles along the support of the target distribution. 
Thus it is essential to choose a kernel that can capture the underlying geometry of the target distribution, so the particles can traverse the support of the target distribution efficiently.
To this end, we can use the curvature information characterised by the Hessian of the logarithm of the target density to design anisotropic kernels.

Consider a positive definite matrix $A(x)$ that approximates the local Hessian of the negative logarithm of the target density, i.e., $A(x) \approx - \nabla_x^2\log\pi(x)$. 
We introduce the metric
\begin{equation}\label{kernel_metric}
	M_\pi \coloneqq \Ev_{x\sim \pi}[A(x)]\,, 
\end{equation}
to characterise the average curvature of the target density, stretching and compressing the parameter space in different directions.
There are a number of computationally efficient ways to evaluate such an $A(x)$---for example, the generalised eigenvalue approach in \cite{stonewton2012} and the Fisher information-based approach in \cite{girolami2011riemann}. 
The expectation in \eqref{kernel_metric} is taken against the target density $\pi$, and thus cannot be directly computed.  
Utilising the ensemble $\{x_i^l\}_{i = 1}^n$ in each iteration, we introduce an alternative metric
\begin{equation}\label{eq:metric_MC}
	M_{p_l} \coloneqq \frac{1}{n} \textstyle \sum_{i = 1}^{n} A(x_i^l),
\end{equation}
to approximate $M_\pi$.
Similar approximations have also been introduced in the context of dimension reduction for statistical inverse problems; see \cite{cui2014likelihood}. 
Note that the computation of the metric \eqref{eq:metric_MC} does not incur extra computational cost, as we already calculated (approximations to) $\nabla_x^2\log\pi(x)$ at each particle in the Newton update.

Given a kernel of the generic form $k(x, x^\prime) = f ( \|x-x^\prime\|^2 )$, we can then use the metric $M_{p_l}$ to define an anisotropic kernel
\[
	k_l(x, x^\prime) = f \bigg( \frac{1}{g(d)}  \|x-x^\prime\|^2_{M_{p_l}} \bigg), 
\]
where the norm $\| \cdot \|_{M_{p_l}}$ is defined as $\|x\|^2_{M_{p_l}} = x^\top M_{p_l} x$ and $g(d)$ is a positive and real-valued function of the dimension $d$.
For example, with $g(d) = d$, the Gaussian kernel used in the SVGD of \cite{liu2016svgd} can be modified as
\begin{equation}\label{kernel_hessian}
	k_l(x,x^\prime) \coloneqq \exp\bigg(-\frac{1}{2d} \|x-x^\prime\|^2_{M_{p_l}} \bigg)\, .
\end{equation}

The metric $M_{p_l}$ induces a deformed geometry in the parameter space: distance is greater along directions where the (average) curvature is large. This geometry directly affects how particles in SVGD or SVN flow---by shaping the locally-averaged gradients and the ``repulsion force'' among the particles---and tends to spread them more effectively over the high-probability regions of $\pi$. 

The dimension-dependent scaling factor $g(d)$ plays an important role in high dimensional problems. 
Consider a sequence of target densities that converges to a limit as the dimension of the parameter space increases. For example, in the context of Bayesian inference on function spaces, e.g., \cite{stuart2010}, the posterior density is often defined on a discretisation of a function space, whose dimensionality increases as the discretisation is refined. 
In this case, the $g(d)$-weighed norm $\|\cdot\|^2/d$ is the square of the discretised $L^2$ norm under certain technical conditions (e.g., the examples in Section \ref{subsec:cond_diff} and the Appendix) and converges to the functional $L^2$ norm as $d \to \infty$. 
With an appropriate scaling $g(d)$, the kernel may thus exhibit robust behaviour with respect to discretisation if the target distribution has appropriate infinite-dimensional limits. 
For high-dimensional target distributions that do not have a well-defined limit with increasing dimension, an appropriately chosen scaling function $g(d)$ can still improve the ability of the kernel to discriminate inter-particle distances. 
Further numerical investigation of this effect is presented in the Appendix.

\section{Test cases}\label{sec:testcases}
We evaluate our new SVN method with the scaled Hessian kernel on a set of test cases drawn from various Bayesian inference tasks. 
For these test cases, the target density $\pi$ is the (unnormalised) posterior density. 
We assume the prior distributions are Gaussian, that is, $\pi_0(x) = \cN(\mpr,\Cpr)$, where $\mpr\in\R^d$ and $\Cpr\in\R^{d\times d}$ are the prior mean and prior covariance, respectively. 
Also, we assume there exists a forward operator $\cF:\R^d\to\R^m$ mapping from the parameter space to the data space. The relationship between the observed data and unknown parameters can be expressed as $y=\cF(x) + \xi$, where $\xi\sim\cN(0, \sigma^2\,I)$ is the measurement error and $I$ is the identity matrix. 
This relationship defines the likelihood function $\cL(y|x) = \cN(\cF(x),\sigma^2\,I)$ and the (unnormalised) posterior density $\pi(x) \propto \pi_0(x)\cL(y|x)$.
  
We will compare the performance of SVN and SVGD, both with the scaled Hessian kernel \eqref{kernel_hessian} and the heuristically-scaled isotropic kernel used in \cite{liu2016svgd}. 
We refer to these algorithms as SVN-H, SVN-I, SVGD-H, and SVGD-I, where `H' or `I' designate the Hessian or isotropic kernel, respectively. Recall that the heuristic used in the `-I' algorithms involves a scaling factor based on the number of particles $n$ and the median pairwise distance between particles \cite{liu2016svgd}.
Here we present two test cases, one multi-modal and the other high-dimensional. In the Appendix, we report on additional tests. First, we evaluate the performance of SVN-H with different Hessian approximations: the exact Hessian (full Newton), the block diagonal Hessian, and a Newton--CG version of the algorithm with exact Hessian. Second, we provide a performance comparison between SVGD and SVN on a high-dimensional Bayesian neural network. Finally, we provide further numerical investigations of the dimension-scalability of our scaled kernel.

\subsection{Two-dimensional double banana}\label{testcase:db}
The first test case is a two-dimensional bimodal and ``banana'' shaped posterior density. The prior is a standard multivariate Gaussian, i.e., $\mpr= 0$ and $\Cpr = I$, and the observational error has standard deviation $\sigma = 0.3$. The forward operator is taken to be a scalar logarithmic Rosenbrock function, i.e.,
\[ \cF(x) = \log\left((1-x_1)^2 + 100(x_2 - x_1^2)^2\right)\,, \]
where $x=(x_1,x_2)$. We take a single observation $y=\cF(\xtrue)+\xi$, with $\xtrue$ being a random variable drawn from the prior and $\xi \sim \cN(0,\sigma^2\,I)$.
  
  \begin{figure*}[h!]
    \centering
		\includegraphics[trim= 4.6cm 9.5cm 4cm 9.5cm, width=0.8\textwidth]{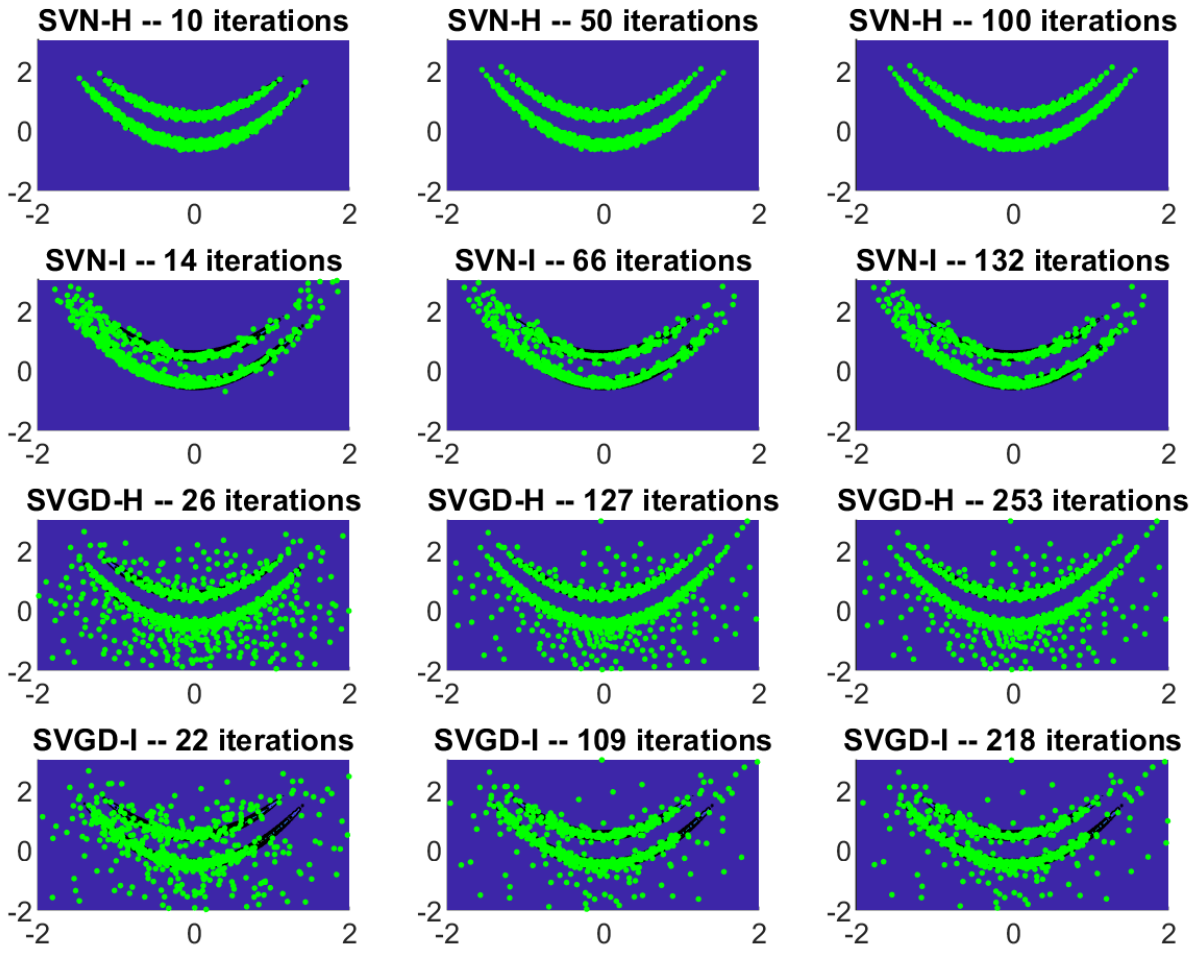}
		\caption{Particle configurations superimposed on contour plots of the double-banana density.}\label{fig:doublebanana}
	\end{figure*}

Figure \ref{fig:doublebanana} summarises the outputs of four algorithms at selected iteration numbers, each with $n=1000$ particles initially sampled from the prior $\pi_0$. 
The rows of Figure \ref{fig:doublebanana} correspond to the choice of algorithms and the columns of Figure \ref{fig:doublebanana} correspond to the outputs at different iteration numbers. 
We run 10, 50, and 100 iterations of SVN-H. To make a fair comparison, we rescale the number of iterations for each of the other algorithms so that the total cost (CPU time) is approximately the same. It is interesting to note that the Hessian kernel takes considerably less computational time than the Isotropic kernel. This is because, whereas the Hessian kernel is automatically scaled, the Isotropic kernel calculates the distance between the particles at each iterations to heuristically rescale the kernel.
  
The first row of Figure \ref{fig:doublebanana} displays the performance of SVN-H, where second-order information is exploited both in the optimisation and in the kernel. After only 10 iterations, the algorithm has already converged, and the configuration of particles does not visibly change afterwards. Here, all the particles quickly reach the high probability regions of the posterior distribution, due to the Newton acceleration in the optimisation. Additionally, the scaled Hessian kernel seems to spread the particles into a structured and precise configuration.
  
The second row shows the performance of SVN-I, where the second-order information is used exclusively in the optimisation. We can see the particles quickly moving towards the high-probability regions, but the configuration is much less structured. After 47 iterations, the algorithm has essentially converged, but the configuration of the particles is noticeably rougher than that of SVN-H.
  
SVGD-H in the third row exploits second-order information exclusively in the kernel. Compared to SVN-I, the particles spread more quickly over the support of the posterior, but not all the particles reach the high probability regions, due to slower convergence of the optimisation. The fourth row shows the original algorithm, SVGD-I. The algorithm lacks both of the benefits of second-order information: with slower convergence and a more haphazard particle distribution, it appears less efficient for reconstructing the posterior distribution. 
  
\subsection{100-dimensional conditioned diffusion}\label{subsec:cond_diff}
The second test case is a high-dimensional model arising from a Langevin SDE, with state $u:[0,T]\to\R$ and dynamics given by
\begin{equation}\label{eq:langevin}
	du_t = \frac{\beta u\, (1-u^2)}{(1+u^2)}\,dt + dx_t,\quad u_0=0\, .
\end{equation}
Here $x=(x_t)_{t\ge 0}$ is a Brownian motion, so that $x\sim \pi_0 = \cN(0,C)$, where $C(t,t') = \min(t,t')$. 
This system represents the motion of a particle with negligible mass trapped in an energy potential, with thermal fluctuations represented by the Brownian forcing; it is often used as a test case for MCMC algorithms in high dimensions \cite{cui2016dimension}. 
Here we use $\beta=10$ and $T=1$. Our goal is to infer the driving process $x$ and hence its pushforward to the state $u$.

\begin{figure*}[h!]
  \centering
	\includegraphics[trim= 1cm 5.8cm 0.5cm 4.5cm, width=0.8\textwidth]{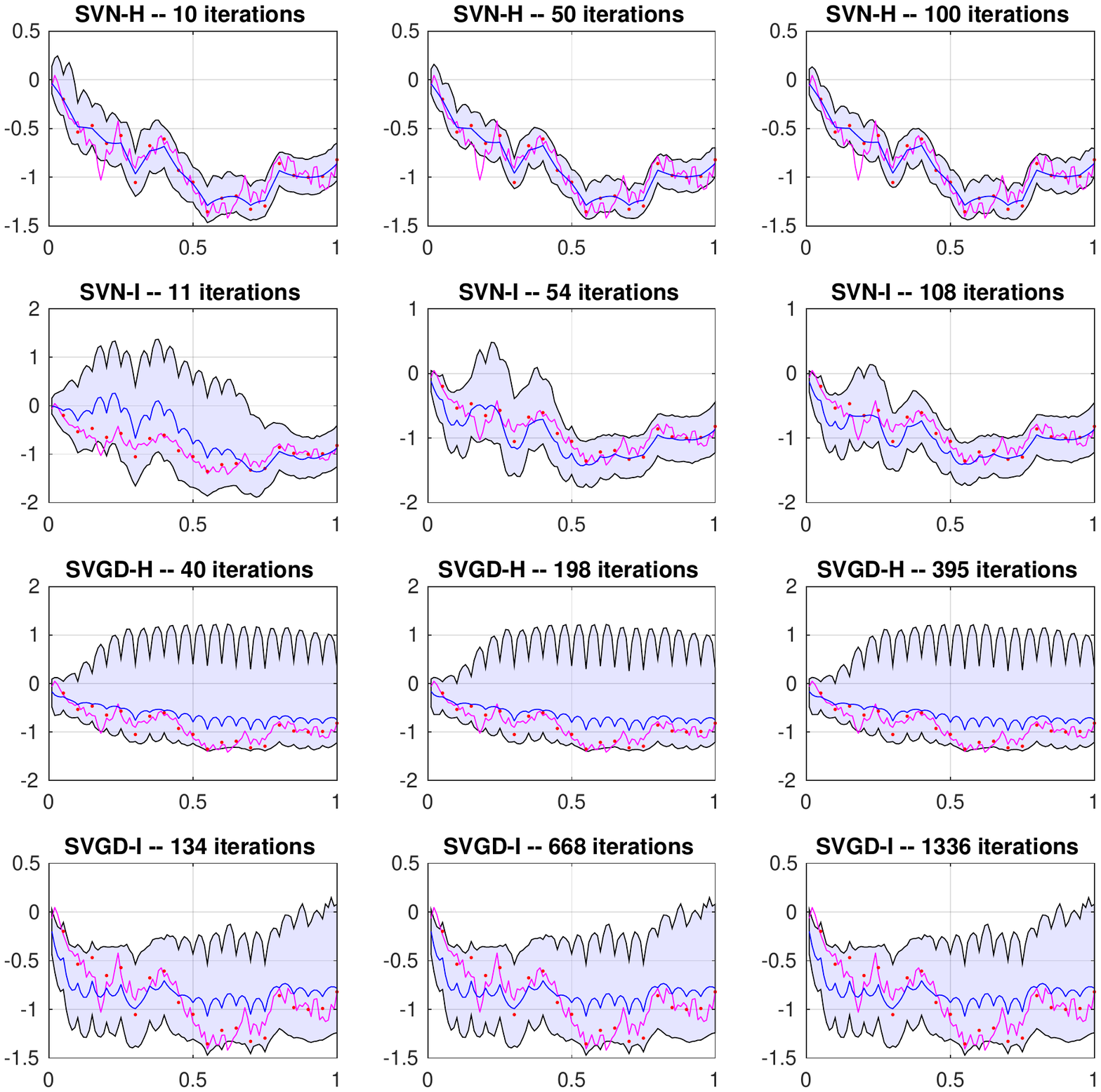}
	\caption{In each plot, the magenta path is the true solution of the discretised Langevin SDE; the blue line is the reconstructed posterior mean; the shaded area is the 90\% marginal posterior credible interval at each time step.}\label{fig:cond_diff}
\end{figure*}
  
The forward operator is defined by $\cF(x) = [u_{t_1}, u_{t_2}, \dots, u_{t_{20}}]^\top\in\R^{20}$, where $t_i$ are equispaced observation times in the interval $(0,1]$, i.e., $t_i = 0.05\,i$. By taking $\sigma=0.1$, we define an observation $y=\cF(\xtrue)+\xi\in\R^{20}$, where $\xtrue$ is a Brownian motion path and $\xi\sim\cN(0,\sigma^2\,I)$. For discretization, we use an Euler-Maruyama scheme with step size $\Delta t = 10^{-2}$; therefore the dimensionality of the problem is $d=100$.
The prior is given by the Brownian motion $x=(x_t)_{t\ge 0}$, described above.

Figure \ref{fig:cond_diff} summarises the outputs of four algorithms, each with $n=1000$ particles initially sampled from $\pi_0$. 
Figure \ref{fig:cond_diff} is presented in the same way as Figure \ref{fig:doublebanana} from the first test case. 
The iteration numbers are scaled, so that we can compare outputs generated by various algorithms using approximately the same amount of CPU time. 
In Figure \ref{fig:cond_diff}, the path in magenta corresponds to the solution of the Langevin SDE in \eqref{eq:langevin} driven by the true Brownian path $\xtrue$. The red points correspond to the 20 noisy observations. The blue path is the reconstruction of the magenta path, i.e., it corresponds to the solution of the Langevin SDE driven by the \textit{posterior mean} of $(x_t)_{t\ge 0}$.  Finally, the shaded area represents the marginal 90\% credible interval of each dimension (i.e., at each time step) of the posterior distribution of $u$.

We observe excellent performance of SVN-H. After 50 iterations, the algorithm has already converged, accurately reconstructing the posterior mean (which in turn captures the trends of the true path) and the posterior credible intervals. 
(See Figure \ref{fig:hmc_comparison} and below for a validation of these results against a reference MCMC simulation.)
SVN-I manages to provide a reasonable reconstruction of the target distribution: the posterior mean shows fair agreement with the true solution, but the credible intervals are slightly overestimated, compared to SVN-H  and the reference MCMC.
The overestimated credible interval may be due to the poor dimension scaling of the isotropic kernel used by SVN-I. 
With the same amount of computational effort, SVGD-H and SVGD-I cannot reconstruct the posterior distribution: both the posterior mean and the posterior credible intervals depart significantly from their true values.

In Figure \ref{fig:hmc_comparison},
we compare the posterior distribution approximated with SVN-H
(using $n=1000$ particles and 100 iterations) to that obtained with 
a reference MCMC run
(using the DILI algorithm of \cite{cui2016dimension} with an effective sample size of $10^5$),
 showing an overall good agreement.
The thick blue and green paths correspond to the posterior means estimated by SVN-H and MCMC,
respectively. The blue and green shaded areas represent the marginal 90\% credible intervals (at each time step) produced by SVN-H and 
MCMC.
In this example, the posterior mean of SVN-H matches that of 
MCMC
quite closely, and both are comparable to the data-generating path
(thick magenta line). (The posterior means are much smoother than the true path, which is to be expected.) The estimated credible intervals of SVN-H and 
MCMC
also match fairly well along the entire path of the SDE.

\begin{figure*}[h!]
  \centering
  \includegraphics[trim= 2cm 6cm 2cm 6cm, 
  width=0.5\textwidth]{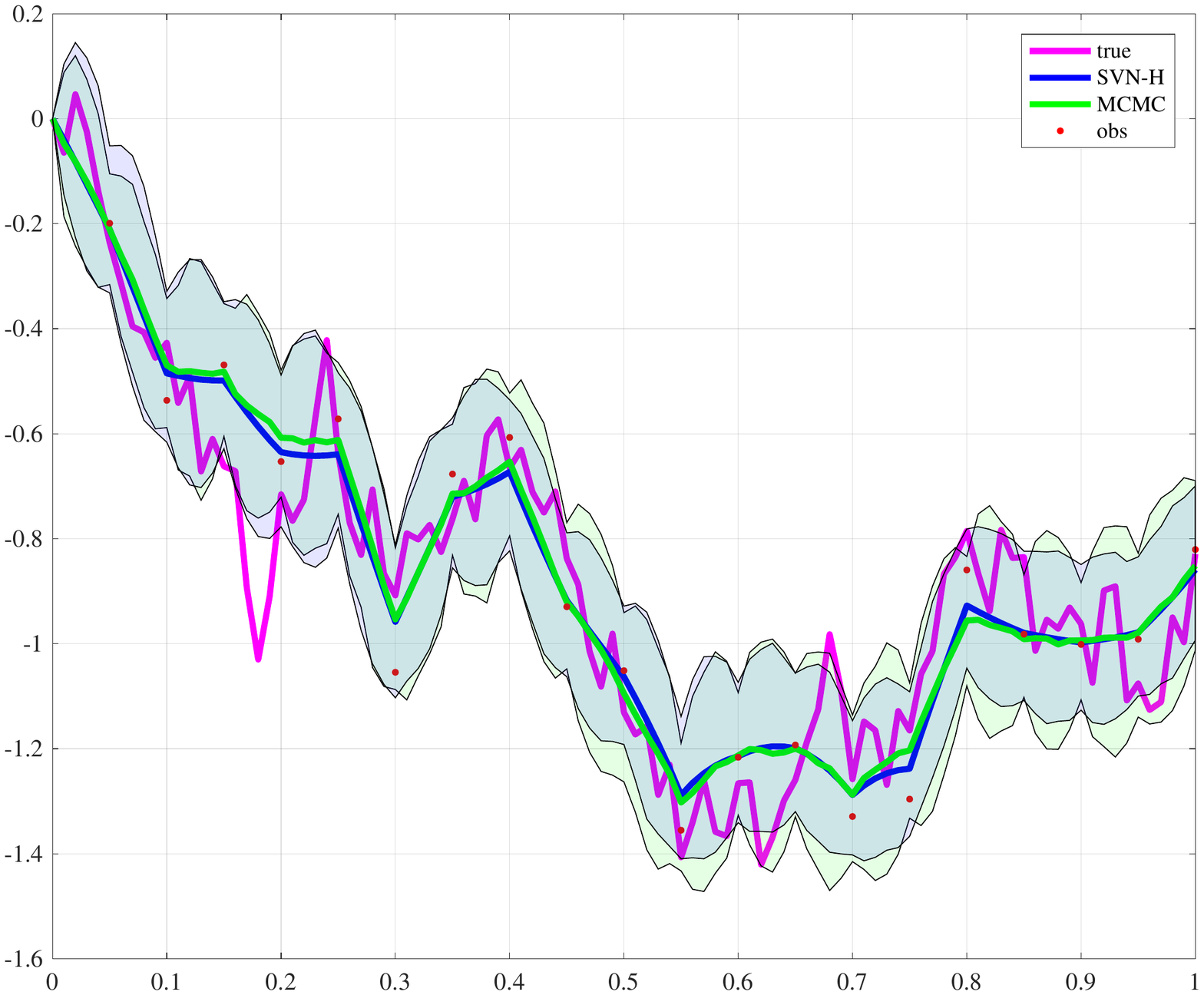} %
  \caption{Comparison of reconstructed distributions from SVN-H and MCMC}\label{fig:hmc_comparison}
\end{figure*}

\section{Discussion}

In general, the use of Gaussian reproducing kernels may be problematic
in high dimensions, due to the locality of the kernel
\cite{francois2005locality}.  While we observe in Section
\ref{sec:kernel} that using a properly rescaled Gaussian kernel can
improve the performance of the SVN method in high dimensions, we also
believe that a truly general purpose nonparametric algorithm using
local kernels will inevitably face further challenges in high-dimensional
settings.  A sensible approach to coping with high dimensionality is also to
design algorithms that can detect and exploit essential {\it
  structure} in the target distribution, whether it be decaying
correlation, conditional independence, low rank, multiple scales, and
so on.  See \cite{spantini2017inference,wang2017structured} for recent
efforts in this direction.

\section{Acknowledgements}
G.\ Detommaso is supported by the EPSRC Centre for Doctoral Training in Statistical Applied Mathematics at Bath (EP/L015684/1) and by a scholarship from the Alan Turing Institute.
T.\ Cui, G.\ Detommaso, A.\ Spantini, and Y.\ Marzouk acknowledge support from the MATRIX Program on ``Computational Inverse Problems'' held at the MATRIX Institute, Australia, where this joint collaboration was initiated.
A.\ Spantini and Y.\ Marzouk also acknowledge support from the AFOSR Computational Mathematics Program.

\appendix

\section{Proof of Theorem 1}\label{apx:proofs}

The following proposition is used to prove Theorem 1.
\begin{proposition}\label{prop:gradient}
	Define the directional derivative of $J_p$ as the first variation of $J_p$ at $S \in \cH^d$ along a direction $V \in \cH^d$, 
	\[
	D J_p[S] (V) := \lim_{\tau \rightarrow 0} \frac{1}{\tau}\big( J_p[S + \tau V] - J_p[S] \big)\,.
	\]
	The first variation takes the form
	\begin{equation}
	D J_p[S] (V) = - \Ev_{x \sim p} \left[  \big( \nabla_x \log\pi(x + S(x))\big)^\top  V(x)  + \textnormal{trace} \big( (I + \nabla_xS(x)) ^{-1}\nabla_x V(x) \big) \right]\,.
	\end{equation}
\end{proposition}
\begin{proof}
	Given the identity map $I$ and a transport map in the form of $T = I + S + \tau V$, the pullback density of $\pi$  is defined as 
	\[
	T^\ast \pi = \pi( T(x)) \,|\det \nabla_x T(x) | = \pi\big( x + S(x) + \tau V(x)\big)\,\big | \det \big( I + \nabla_x S(x) + \tau \nabla_xT(x) \big) \big|\,. 
	\]
	The perturbed  objective function $J_p[S + \tau V]$ takes the form 
	\begin{align*}
	J_p[S + \tau V] &= \KL((I + S + \tau V)_\ast p\ \|\ \pi) \notag\\ 
	&= \KL( p\ \|\ (I + S + \tau V)^\ast\pi) \notag \\
	&= \int p(x) \log p(x) dx - \int p(x)\Big( \log \pi\big( x + S(x) + \tau V(x)\big) \notag  \\
	&\quad\quad + \log \big |\det\big( I + \nabla_x S(x) + \tau \nabla_x V(x) \big) \big| \Big)\,dx \,. %
	\end{align*}
	Thus we have
	\begin{align}
	J_p[S + \tau V] &- J_p[S] =   - \int p(x)\bigg( \underbrace{\log \pi\big( x + S(x) + \tau V(x)\big) - \log \pi(x + S(x)) }_{(i)} \bigg) \,dx \notag \\ 
	& - \int p(x) \big( \underbrace{ \log \big |\det\big( I + \nabla_x S(x) + \tau \nabla_x V(x) \big) \big| - \log \big |\det\big( I + \nabla_x S(x) \big) \big| }_{(ii)}  \big) \,dx \,. \label{eq:dJdeps}
	\end{align}
	Performing a Taylor expansion of the terms (i) and (ii) in \eqref{eq:dJdeps}, 
	we have
	\begin{align*}
	(i) & = \tau \big( \nabla_x \log\pi(x + S(x)) \big)^\top V(x) + O(\tau^2)\, ,\\
	(ii) & = \tau\, \trace \big( (I + \nabla_x S(x) )^{-1}\nabla_x V(x) \big) + O(\tau^2)\, ,
	\end{align*}
	where $\nabla_x \log\pi(x + S(x))$ is the partial derivative of $\log \pi$ evaluated at $x + S(x)$. 
	Plugging the above expression into \eqref{eq:dJdeps} and the definition of the directional derivative, we obtain
	\begin{equation}
	D J_p[S] (V) = - \Ev_{x \sim p} \left[  \big( \nabla_x \log\pi(x + S(x))\big)^\top  V(x)  + \trace \big( \nabla_x (x + \nabla_x S(x) )^{-1}\nabla_x V(x) \big) \right].
	\end{equation}
\end{proof}

The Fr\'{e}chet derivative of $J_p$ evaluated at  $S \in \cH^d$, $\nabla J_p[S]: \cH^d \to \cL(\cH^d, \R)$ satisfies 
\[
D J_p[S] (V) = \langle \nabla J_p[S]\,, V \rangle_{\cH^d}, \quad \forall \, V \in \cH^d\, ,
\]
and thus we can use Proposition \ref{prop:gradient} to prove Theorem 1.

\begin{proof}[Proof of Theorem 1]
	The second variation of $J_p$ at $\bs 0$ along directions $V, W \in \cH^d$ takes the form
	\[
	D^2 J_p[\bs 0] (V, W) := \lim_{\tau \rightarrow 0} \frac{1}{\tau}\big( DJ_p[\tau W](V) - DJ_p[\bs 0](V) \big)\,.
	\]
	Following Proposition 3, we have 
	\begin{align}
	D^2 J_p[\bs 0] (V, W) = & \lim_{\tau \rightarrow 0} \frac{1}{\tau} \big( D J_p[\tau W] (V) - D J_p[\bs 0] (V) \big ) \notag\\
	= &  - \Ev_{x \sim p} \big[  \underbrace{ \lim_{\tau \rightarrow 0}  \frac{1}{\tau}  \big( \nabla_x \log\pi(x + \tau W(x)) - \nabla_x \log\pi(x) }_{(i)} \big)^\top  V(x) \big] \notag \\
	& - \Ev_{x \sim p} \big[ \trace \big( \underbrace{ \lim_{\tau \rightarrow 0} \frac{1}{\tau}  [ (I + \tau \nabla_x W(x) )^{-1} - I ]}_{(ii)} \nabla_x V(x) \big) \big]\,.
	\end{align}
	
	By Taylor expansion, the %
	limits (i) and (ii) of the above equation can be written as %
	\begin{align*}
	(i) & = \nabla_x^2 \log\pi(x) W(x)\, ,\\
	(ii) & = - \nabla_x W(x)\, .
	\end{align*}
	Thus, the second variation of $J_p$ at $\bs 0$ along directions $V, W \in \cH^d$ becomes 
	\begin{equation}\label{eq:der2}
	D^2 J_p[\bs 0] (V, W) = - \Ev_{x \sim p} \big[ W(x)^\top \nabla_x^2 \log\pi(x) V(x) - \trace \big( \nabla_x W(x) \nabla_x V(x) \big) \big]\,.
	\end{equation}
	Using the reproducing property of $V,W \in \cH^d$, i.e.
	\begin{align*}
	&v_i(x) = \langle k(x, \cdot), v_i(\cdot) \rangle_{\cH}\, , &&w_j(x) = \langle k(x, \cdot), w_j(\cdot) \rangle_{\cH}\\
	&\nabla_x v_i(x) = \langle \nabla_x k(x, \cdot), v_i(\cdot) \rangle_{\cH^d}\,, &&\nabla_x w_i(x) = \langle \nabla_x k(x, \cdot), w_i(\cdot) \rangle_{\cH^d}
	\end{align*}
	we then have
	\[
	\Ev_{x \sim p} \big[ W(x)^\top \nabla_x^2 \log\pi(x) V(x) \big] = \sum_{i=1}^d\sum_{j=1}^d \Big\langle\langle\Ev_{x \sim p} \big[\de_{ij}^2\log\pi(x) k(x,y) k(x,z)\big], w_j(z)\rangle_{\cH}, v_i(y)\Big\rangle_{\cH}
	\]
	and
	\[
	\Ev_{x \sim p} \big[\trace \big( \nabla_x W(x) \nabla_x V(x) \big) \big] = \sum_{i=1}^d\sum_{j=1}^d \Big\langle\langle\Ev_{x \sim p} \big[\de_i k(x,y)\de_j k(x,z)\big], w_j(z)\rangle_{\cH}, v_i(y)\Big\rangle_{\cH}\,.
	\]
	Plugging the above identities into \eqref{eq:der2}, the second variation can be expressed as
	\[
	D^2 J_p[\bs 0] (V, W) = \sum_{i=1}^d\sum_{j=1}^d \Big\langle\langle h_{ij}(y,z), w_j(z)\rangle_{\cH}, v_i(y)\Big\rangle_{\cH}\,,
	\]
	where 
	\[
	h_{ij}(y,z) := \Ev_{x \sim p} \big[-\de_{ij}^2\log\pi(x) k(x,y) k(x,z)+\de_i k(x,y)\de_j k(x,z)\big]\,.
	\]	
	Hence the result.
\end{proof}

\section{Proof of Corollary 1}

\begin{proof}
	Here we drop the subscript $p_l$. The ensemble of particles $(x_k)_{k=1}^n$ defines a linear function space  $\cH_n = {\rm span}\{ k(x_1, \cdot), \ldots, k(x_n, \cdot) \}$.
	In the Galerkin approach, we seek a solution $W = (w_1,\dots,w_d)^\top \in \cH^d_n$ such that the residual of the Newton direction
	\begin{equation}\label{eq:newton}
	\sum_{i=1}^d\left\langle
	\sum_{j=1}^d
	\left\langle  h_{ij}(y,z), w_j(z)\right\rangle_{\cH} + \partial_i J[\bs 0](y), v_i(y)
	\right\rangle_{\cH} = 0, 
	\end{equation}
	is zero for all possible $V \in \cH^d_n$.
	This way, we can approximate each component $w_j$ of the function $W$ as
	\begin{equation} \label{eq:wi}
	w_j(z) = \sum_{k=1}^n \alpha^k_j \,k(x_k, z),
	\end{equation}
	for a collection of unknown coefficients $(\alpha^k_j)$.
	We define $V^s = (v_1^s,\dots,v_d^s)^\top$ to be the test function where $v^s_i(y) = k(x_s, y)$ for all $s = 1, \ldots, n$.

	We first project the Newton direction \eqref{eq:newton} onto $V^s$ for all $s = 1, \ldots, n$. Applying the reproducing property of the kernel, this leads to 
	\begin{equation}  \label{eq:semi-discr}
	\sum_{j=1}^d
	\left\langle  h_{ij}(x_s, z), w_j(z)\right\rangle_{\cH^d} + \partial_i J_{p_l}[\bs 0](x_s) = 0,
	\qquad i=1,\ldots,d,\quad s=1,\ldots,n.
	\end{equation}
	Plugging \eqref{eq:wi}  into \eqref{eq:semi-discr}, we obtain the fully discrete set of equations
	\begin{equation}  \label{eq:discr}
	\sum_{j=1}^d\, \sum_{\ell=1}^n \,
	h_{ij}(x_s,x_k)\,\alpha^k_j  + \partial_i J_{p_l}[\bs 0](x_s) = 0,
	\quad i=1,\ldots,d,\; s=1,\ldots,n, \; k=1,\ldots,n.
	\end{equation}
	We denote the coefficient vector $\alpha^k := \big(\alpha^k_1, \ldots, \alpha^k_d\big)^\top$ for each $x_k$, the block Hessian matrix $(H^{s,k} )_{ij} := h_{ij}(x_s,x_k)$ for each pair of $x_s$ and $x_k$, and $\nabla J^s := \nabla J[\bs 0](x_s)$ for each $x_s$. 
	Then equation \eqref{eq:discr} can be expressed as
	\begin{equation}  \label{eq:semi-comp}
	\sum_{k=1}^n \,  
	H^{s,k}\,\alpha^k   = \nabla J^s,
	\qquad s=1,\ldots,n.
	\end{equation}
\end{proof}

\section{Additional test cases}\label{apx:testcases}

\subsection{Comparison between the full and inexact Newton methods}

Here we compare three different Stein variational Newton methods: \texttt{SVNfull} denotes the method that solves the fully coupled Newton system in equation (16) of the main paper, with no approximations; \texttt{SVNCG} denotes the method that applies inexact Newton--CG to the fully coupled system (16); and \texttt{SVNbd} employs the block-diagonal approximation given in equation (17) of the main paper.

We first make comparisons using the two-dimensional double banana distribution presented in Section 5.1. 
We run our test case for $N=100$ particles and 20 iterations. Figure \ref{fig:compnewtonsdb} shows the contours of the target density and the samples produced by each of the three algorithms. Compared to the full Newton method, both the block-diagonal approximation and the inexact Newton--CG generate results of similar quality.

\begin{figure*}[!h]
	\centering
	\includegraphics[trim= 3cm 10cm 3cm 9cm, 
	width=0.8\textwidth]{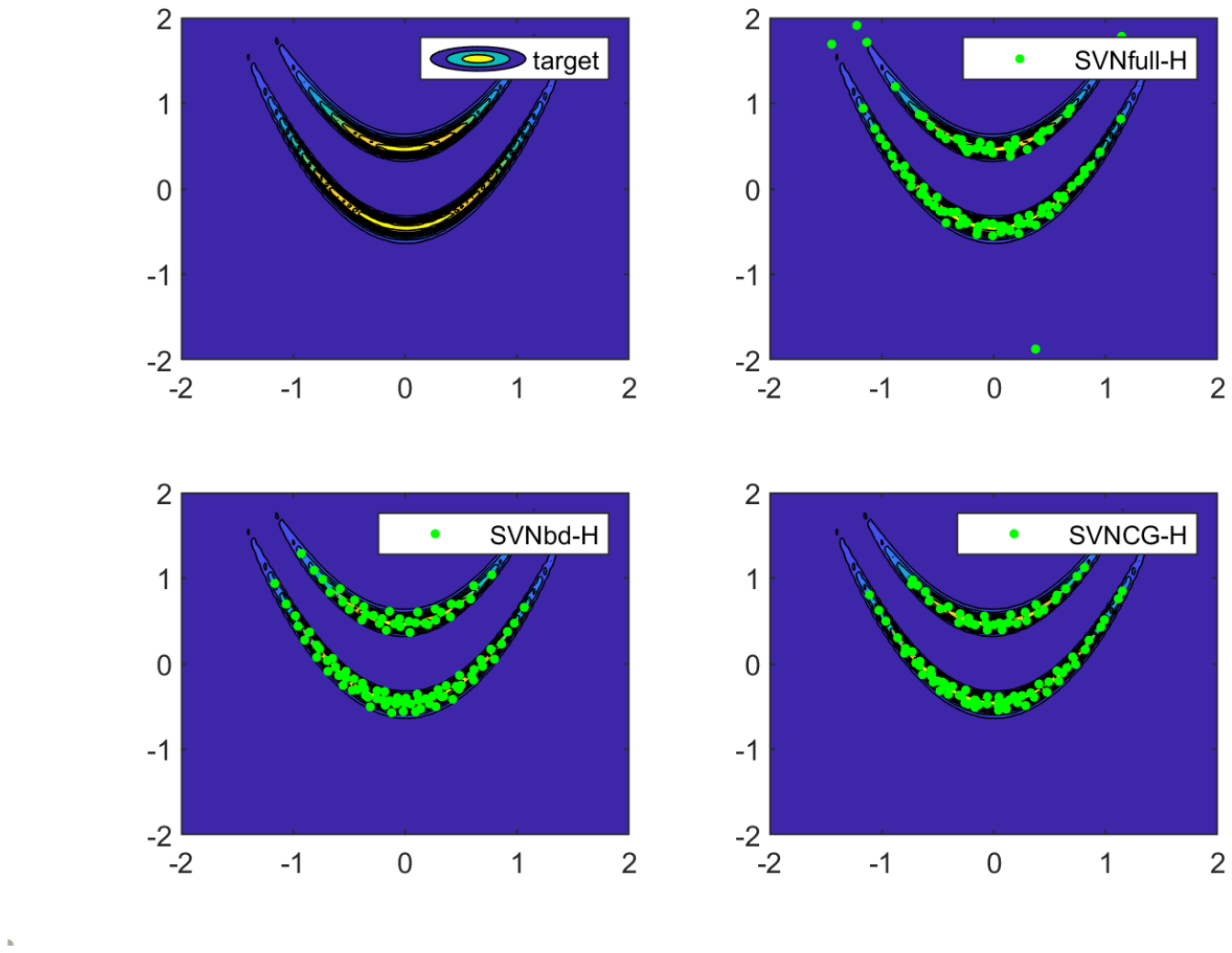} %
	\caption{Double-banana example: performance comparison between \texttt{SVNfull}, \texttt{SVNCG}, and \texttt{SVNbd} after 20 iterations}\label{fig:compnewtonsdb}
\end{figure*}


We use an additional nonlinear regression test case for further comparisons. In this case, the forward operator is given by
\[ \mathcal{F}(x) = c_1 x_1^3 + c_2 x_2\,, \]
where  $x = [x_1,x_2]^\top$ and $c_1,c_2$ are some fixed coefficients sampled independently from a standard normal distribution. A data point is then given by $y=\mathcal{F}(x)+\varepsilon$, where $\varepsilon\sim N(0,\sigma^2)$ and $\sigma = 0.3$. We use a standard normal prior distribution on $x$.

\begin{figure*}[h!]
	\centering
	\includegraphics[trim= 3cm 10cm 3cm 9cm, 
	width=0.8\textwidth]{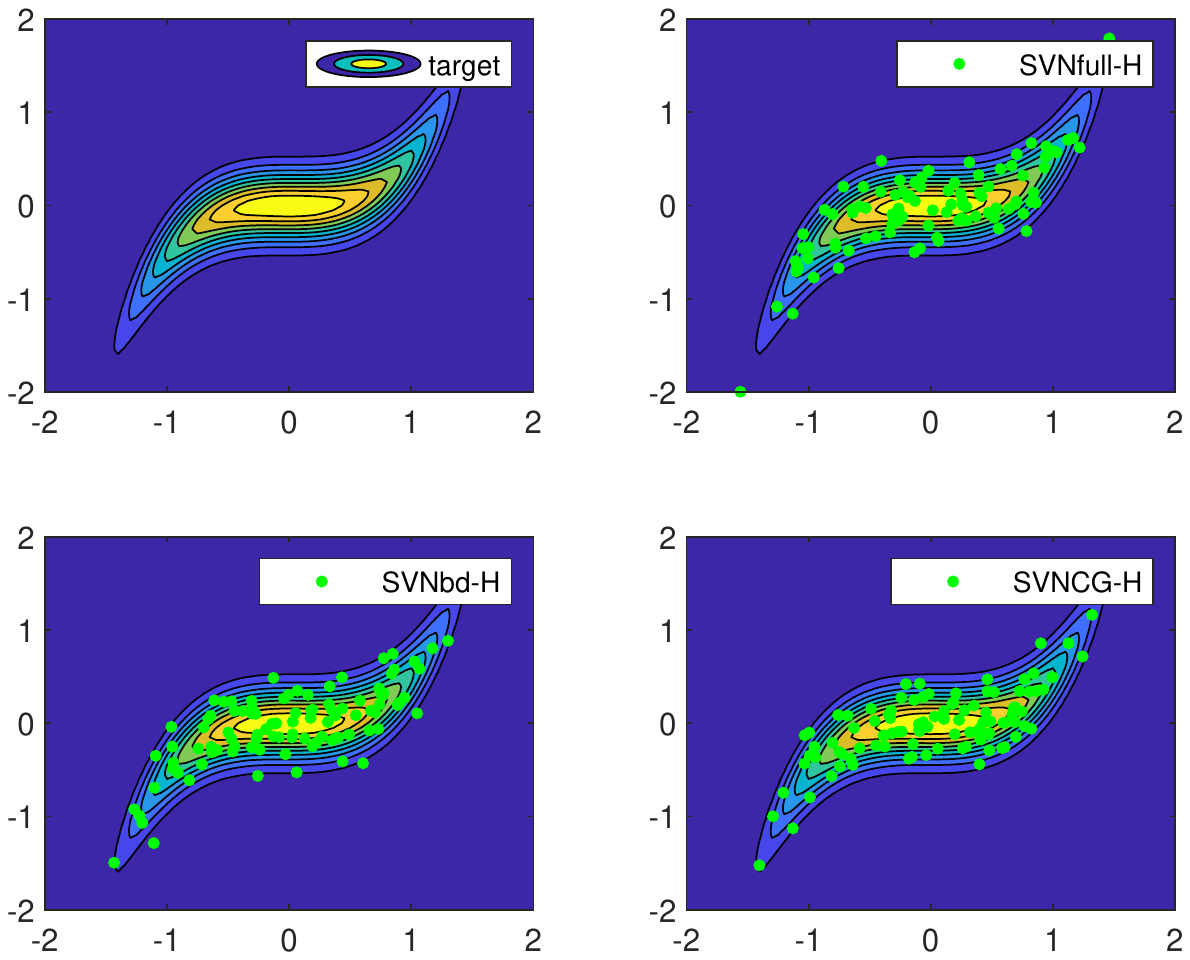} %
	\caption{Nonlinear regression example: performance comparison between \texttt{SVNfull}, \texttt{SVNCG}, and \texttt{SVNbd} after 20 iterations}\label{fig:nonlinregr}
\end{figure*}

We run our test case for $N=100$ particles and 20 iterations. Figure \ref{fig:nonlinregr} shows contours of the posterior density and the samples produced by each of the three algorithms.
Again, both the block-diagonal approximation and the inexact Newton--CG generate results of similar quality to those of the full Newton method.

These numerical results suggest that the block-diagonal approximation and the inexact Newton--CG can be effective methods for iteratively constructing the transport maps in SVN. We will adopt these approximate SVN strategies on large-scale problems, where computing the full Newton direction is not feasible.

\subsection{Bayesian neural network}

In this test case, we set up a Bayesian neural network as described in \cite{liu2016svgd}. We use the open-source ``yacht hydrodynamics'' data set\footnote{\url{http://archive.ics.uci.edu/ml/datasets/yacht+hydrodynamics}} and denote the data by $\mathcal{D}=(x_i,y_i)_{i=1}^{M}$, where $x_i$ is an input, $y_i$ is the corresponding scalar prediction, and $M=308$. We divide the data into a training set of $m=247$ input--prediction pairs and a validation set of $M-m = 61$ additional pairs. 
For each input, we model the corresponding prediction as
\[ y_i = f(x_i,w) + \varepsilon_i\, , \]  
where $f$ denotes the neural network with weight vector $w \in \mathbb{R}^d$ and $\varepsilon_i\sim N(0,\gamma^{-1})$ is an additive Gaussian error. The dimension of the weight vector is $d = 2951$.  We endow the weights $w$ with independent Gaussian priors, $w \sim N(0,\lambda^{-1} I)$. The inference problem then follows from the likelihood function,
\[ \cL( \mathcal{D} | w, \gamma) = \left(\frac{\gamma}{2\pi}\right)^{\frac{m}{2}}\exp\left(-\frac{\gamma}{2}\sum_{i=1}^m(f(x,w) - y_i)^2\right), \]
and the prior density,
\[ \pi_0(w|\lambda) = \left(\frac{\lambda}{2\pi}\right)^{\frac{m}{2}}\exp\left(-\frac{\gamma}{2}\sum_{i=1}^m w_j^2\right), \, \]
where $\gamma$ and $\lambda$ play the role of hyperparameters. 

\paragraph{Performance comparison of SVN-H with SVGD-I.} 

We compare SVN-H with the original SVGD-I algorithm on this Bayesian neural network example, with hyperparameters fixed to $\log \lambda  = -10$ (which provides a very uninformative prior distribution) and $\log \gamma = 0$. First, we run a line-search with Newton--CG to find the posterior mode $w^\ast$. Figure \ref{fig:mapprediction} shows that neural network predictions at the posterior mode almost perfectly match the validation data. 
\begin{figure*}[h!]
	\centering
	\includegraphics[trim= 2cm 7cm 2cm 6cm, 
	width=0.8\textwidth]{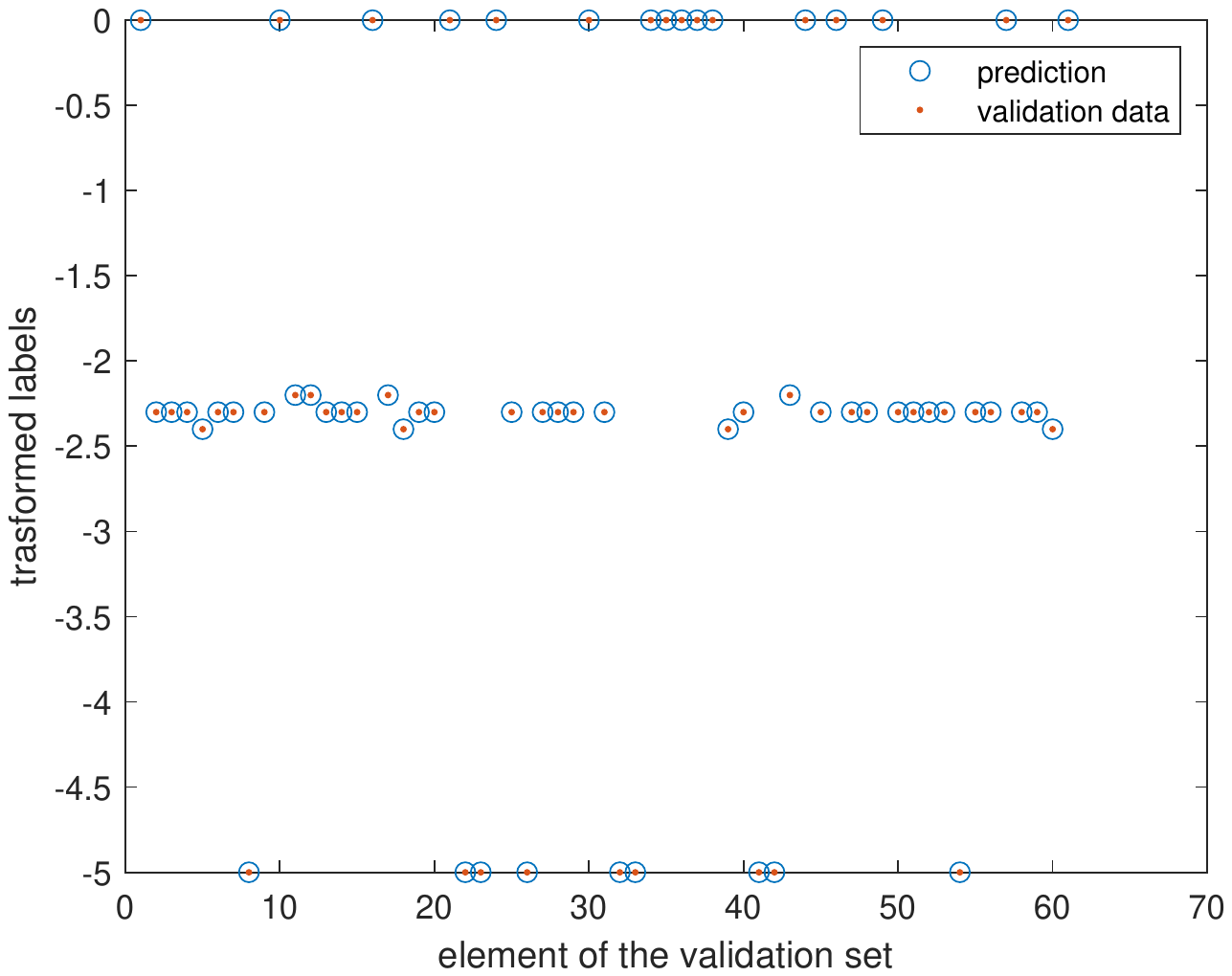} %
	\caption{Neural network prediction at the posterior mode very closely matches the validation data.}\label{fig:mapprediction}
\end{figure*}
Then, we randomly initialise $n=30$ particles $(x_i)_{i=1}^n$ around the mode, i.e., by independently drawing $x_i\sim \cN(w^\ast, I)$. As in the previous test cases, we make a fair comparison of SVN-H and SVGD-I by taking 10, 20, and 30 iterations of SVN-H and rescaling the number of iterations of SVGD-I to match the computational costs of the two algorithms. Because this test case is very high-dimensional, rather than storing the entire Hessian matrix and solving the Newton system we use the inexact Newton--CG approach within SVN, which requires only matrix-vector products and yields enormous memory savings. Implementation details can be found in our GitHub repository.

\begin{figure*}[!h]
	\centering
	\includegraphics[trim= 3cm 7cm 2cm 8cm, 
	width=0.8\textwidth]{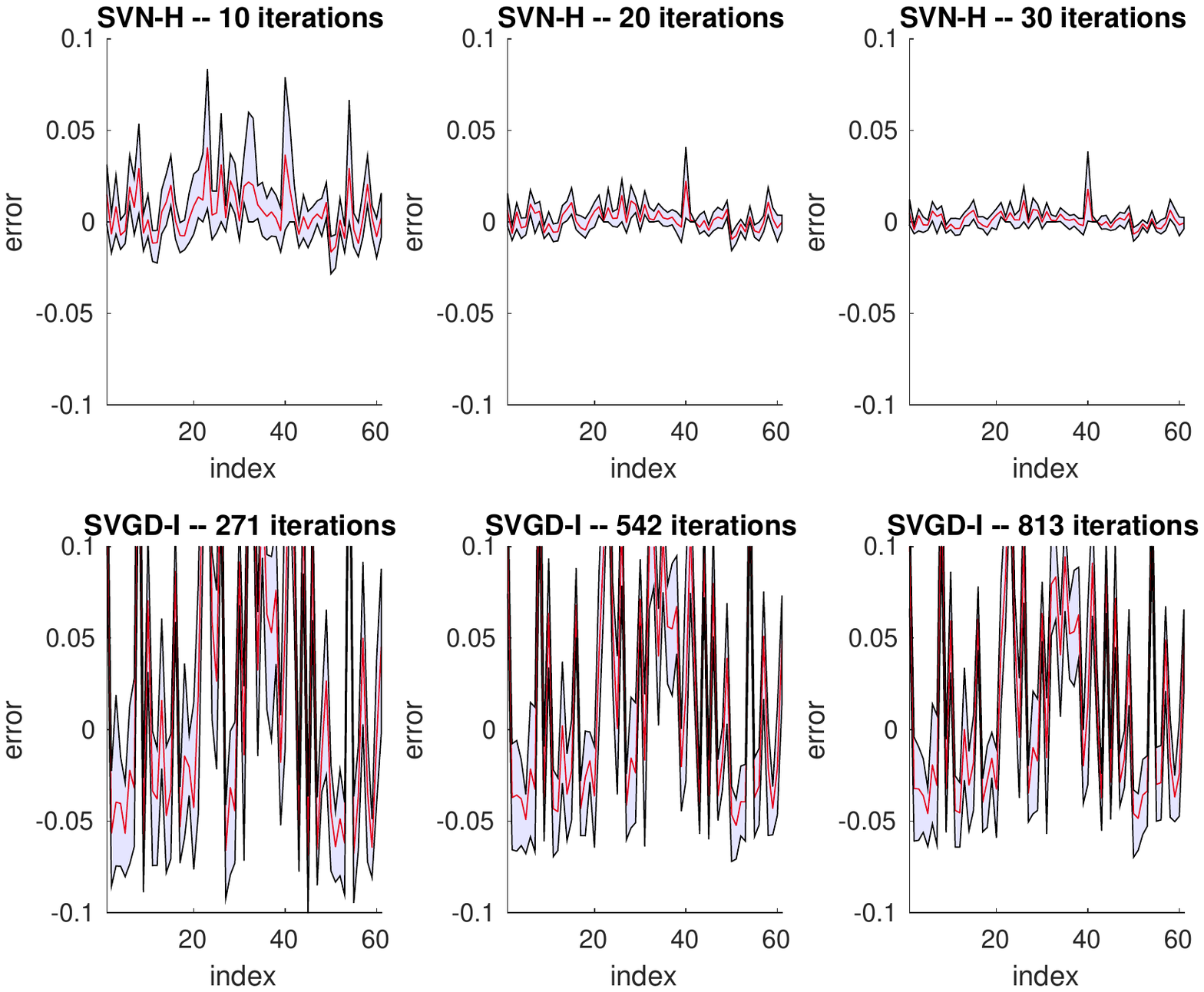} %
	\caption{Bayesian neural network example: Comparison between SVN-H and SVGD-I, showing the distribution of errors between the validation data and samples from the posterior predictive. }\label{fig:bnncomparison}
\end{figure*}

Figure \ref{fig:bnncomparison} shows distributions of the error on the validation set, as resulting from posterior predictions. To obtain these errors, we use the particle representation of the posterior on the weights $w$ to evaluate posterior predictions on the validation inputs $(x_i)_{i=m+1}^M$. Then we evaluate the error of each of these predictions. The red line represents the mean of these errors at each validation input $x_i$, and the shaded region represents the 90\% credible interval of these error distribution. Although both algorithms ``work'' in the sense of producing errors of small range overall, SVN-H yields distributions of prediction error with smaller means and considerably reduced variances, compared to SVGD-I.

\subsection{Scalability of kernels in high dimensions}\label{subsec:kernelperformance}

\paragraph{Discretization-invariant posterior distribution.}
Here we illustrate the dimension scalability of the scaled Hessian kernel, compared to the isotropic kernel used in \cite{liu2016svgd}.
We consider a linear Bayesian inverse problem in a function space setting \cite{stuart2010}: 
the forward operator is a linear functional $\cF(x) = \langle \sin(\pi s), x(s) \rangle$, where the function $x$ is defined for $s \in [0,1]$. The scalar observation $y = \cF(x) + \xi$, where $\xi$ is Gaussian with zero mean and standard deviation $\sigma = 0.3$. The prior is a Gaussian measure $\cN(0,\cK^{-1})$ where $\cK$ is the Laplace operator $-x''(s), \ s \in [0, 1]$, with zero essential boundary conditions. 

Discretising this problem with finite differences on a uniform grid with $d$ degrees of freedom, we obtain a Gaussian prior density $\pi_0(x)$ with zero mean and covariance matrix $K^{-1}$, where $K$ is the finite difference approximation of the Laplacian. 
Let the vector $a$ denote the discretised function $\sin(\pi s), s \in [0, 1]$, and let the corresponding discretised parameter be denoted by $x$ (overloading notation for convenience). Then the finite-dimensional forward operator can be written as $\cF(x) = a^\top x$.
After discretization, the posterior has a Gaussian density of the form $\pi=\cN(\mpos,\Cpos)$, where 
\[ 
\mpos = \frac{y}{\sigma^2}\Cpos\, a\,,\quad\quad\quad \Cpos = \big(K^{-1}+\frac{1}{\sigma^2} a a^\top \big)^{-1}\,. 
\]

To benchmark the performance of various kernels, we construct certain summaries of the posterior distribution. In particular, we use our SVN methods with the scaled Hessian kernel (SVN-H) and the isotropic kernel (SVN-I) to estimate the component-wise average of the posterior mean, $\tfrac{1}{d}\sum_{i=1}^d m_{\textnormal{pos}, i}$, and the trace of the posterior covariance, $\trace(\Cpos)$, for problems discretised at different resolutions $d \in \{40, 60, 80, 100\}$. 
We run each experiment with $n=1000$ particles and $50$ iterations of SVN. 
We compare the numerical estimates of these quantities to the analytically known results. 
These comparisons are summarised in Tables \ref{table_Q1} and \ref{table_Q2}.
From Table \ref{table_Q1}, we can observe that all algorithms almost perfectly recover the average of the posterior mean up to the first three significant figures. 
However, Table \ref{table_Q2} shows that SVN-H does a good job in estimating the trace of the posterior covariance consistently for all dimensions, whereas SVN-I under-estimates the trace---suggesting that particles are under-dispersed and not correctly capturing the uncertainty in the parameter $x$.
This example suggests that the scaled Hessian kernel can lead to a more accurate posterior reconstruction for high-dimensional distributions than the isotropic kernel.

\begin{table}[h!]
	\caption{Comparison of theoretical and estimated averages of the posterior mean}
	\label{table_Q1}
	\centering
	\begin{tabular}{rcccc}
		\toprule
		\multicolumn{5}{c}{Averages of the posterior mean $\tfrac{1}{d}\,\sum_{i=1}^d m_{\textnormal{pos}, i}$}                   \\
		\cmidrule(r){1-5}
		$d$         & 40    & 60    & 80    & 100  \\
		\midrule
		Theoretical &   0.4658  &  0.4634  &  0.4622  &  0.4615   \\
		SVN-H  &  0.4658  &  0.4634  &  0.4623  &  0.4614   \\
		SVN-I & 0.4657  &  0.4633  &  0.4622  &  0.4615   \\
		\bottomrule
	\end{tabular}
\end{table}

\begin{table}[h!]
	\caption{Comparison of theoretical and estimated traces of the posterior covariance}
	\label{table_Q2}
	\centering
	\begin{tabular}{rcccc}
		\toprule
		\multicolumn{5}{c}{Traces of the posterior covariance $\trace(\Cpos)$}    \\
		\cmidrule(r){1-5}
		$d$         & 40    & 60    & 80    & 100  \\
		\midrule
		Theoretical &  0.1295  &  0.1297  &  0.1299  &  0.1299   \\
		SVN-H  & 0.1271  &  0.1281  &  0.1304  &  0.1293    \\
		SVN-I  & 0.0925  &  0.0925  &  0.0925  &  0.0923   \\
		\bottomrule
	\end{tabular}
\end{table}

\paragraph{A posterior distribution that is not discretization invariant.}  Now we examine the dimension-scalability of various kernels in a problem that does not have a well-defined limit with increasing parameter dimension.
We modify the linear Bayesian inverse problem introduced above: now the prior covariance is the identity matrix, i.e., $K^{-1} = I$ and the vector $a$ used to define the forward operator is drawn from a uniform distribution, $a_i \sim \mathcal{U}(2,10), \ i = 1, \ldots, d$. This way, the posterior is not discretization invariant. 
We perform the same set of numerical experiments as above and summarise the results in Tables \ref{table_Q3} and \ref{table_Q4}.
Although the target distribution used in this case is not discretization invariant, 
the scaled Hessian kernel is still reasonably effective in reconstructing the target distributions of increasing dimension (according to the summary statistics below), whereas the isotropic kernel under-estimates the target variances for all values of dimension $d$ that we have tested.

\begin{table}[h!]
	\caption{Comparison of theoretical and estimated averages of the posterior mean}
	\label{table_Q3}
	\centering
	\begin{tabular}{rccccc}
		\toprule
		\multicolumn{5}{c}{Averages of the posterior mean $\tfrac{1}{d}\,\sum_{i=1}^d m_{\textnormal{pos}, i}$}   \\
		\cmidrule(r){1-5}
		$d$         & 40    & 60    & 80    & 100  \\
		\midrule
		Theoretical &   0.0037  &  0.0025  &  0.0019  &  0.0015   \\
		SVN-H  & 0.0037  &  0.0025  &  0.0019  &  0.0015   \\
		SVN-I  & 0.0037  &  0.0025  &  0.0019  &  0.0015   \\
		\bottomrule
	\end{tabular}
\end{table}
\begin{table}[!h]
	\caption{Comparison of theoretical and estimated traces of the posterior covariance}
	\label{table_Q4}
	\centering
	\begin{tabular}{rccccc}
		\toprule
		\multicolumn{5}{c}{Traces of the posterior covariance $\trace(\Cpos)$}  \\
		\cmidrule(r){1-5}
		$d$         & 40    & 60    & 80    & 100  \\
		\midrule
		Theoretical & 39.0001  & 59.0000 &  79.0000 &  99.0000   \\
		SVN-H  & 37.7331 &  55.8354 &  73.6383 &  90.7689    \\
		SVN-I  & 8.7133  &  8.2588  &  7.9862  &  7.6876   \\
		\bottomrule
	\end{tabular}
\end{table}

%
%
%
%
%
%
%

%
%
%
%
%
%
%
%

%
%
%
%
%
%
%
%

%
%
%
%
%
%
%
%
%
%
%
%
%
%
%
%
%
%
%
%
%
%
%
%
%
%
%
%
%
%

%
%
%
%
%

%
%
%

%
%
%

%


\bibliographystyle{abbrv}

\begin{thebibliography}{10}
\bibitem{github}
\url{http://github.com/gianlucadetommaso/Stein-variational-samplers}

\bibitem{AnderesCoram2012}
E.~Anderes, M.~Coram.
\newblock A general spline representation for nonparametric and semiparametric density estimates using diffeomorphisms.
\newblock {\em arXiv preprint arXiv:1205.5314}, 2012.

\bibitem{aronszajn1950theory}
N.~Aronszajn.
\newblock Theory of reproducing kernels.
\newblock {\em Transactions of the American mathematical society}, 
p.~337--404, 1950.
%
\bibitem{blei2017variational}
D.~Blei, A.~Kucukelbir, and J.~D.~McAuliffe.
\newblock Variational inference: A review for statisticians.
\newblock {\em Journal of the American Statistical Association}, 
p.~859--877, 2017.
%
\bibitem{chen2018steinpoints}
W.~Y.~Chen, L.~Mackey, J.~Gorham, F.~X.~Briol, C.~J.~Oates.
\newblock Stein points.
\newblock In {\em International Conference on Machine Learning}.
\newblock {\em arXiv:1803.10161}, 2018.

\bibitem{cui2016dimension}
T.~Cui, K.~J.~H.~Law, Y.~M.~Marzouk.
\newblock Dimension-independent likelihood-informed MCMC.
\newblock {\em Journal of Computational Physics}, 304: 109--137, 2016.

\bibitem{cui2014likelihood}
T.~Cui, J.~Martin, Y.~M.~Marzouk, A.~Solonen, and A.~Spantini.
\newblock Likelihood-informed dimension reduction for nonlinear inverse
problems.
\newblock {\em Inverse Problems}, 30(11):114015, 2014.
%

%

%
\bibitem{francois2005locality}
D.~Francois, V.~Wertz, and M.~Verleysen.
\newblock About the locality of kernels in high-dimensional spaces.
\newblock {\em International Symposium on Applied Stochastic Models and Data Analysis}, 
p.~238--245, 2005.
%
\bibitem{gershman2012nonparvi}
S.~Gershman, M.~Hoffman, D.~Blei.
\newblock Nonparametric variational inference.
\newblock {\em arXiv preprint arXiv:1206.4665}, 2012.

\bibitem{gilks1995markov}
W.~R.~Gilks, S.~Richardson, and D.~Spiegelhalter.
\newblock Markov chain Monte Carlo in practice.
\newblock {\em CRC press}, 1995.

\bibitem{girolami2011riemann}
M.~Girolami and B.~Calderhead.
\newblock Riemann manifold Langevin and Hamiltonian Monte Carlo methods.
\newblock {\em Journal of the Royal Statistical Society: Series B (Statistical
Methodology)}, 73(2):123--214, 2011.
%
\bibitem{han2017adapt}
J.~Han and Q.~Liu.
\newblock Stein variational adaptive importance sampling.
\newblock {\em arXiv preprint arXiv:1704.05201}, 2017.
%

%

%
\bibitem{khal2017viRMSprop}
M.~E.~Khan, Z.~Liu, V.~Tangkaratt, Y.~Gal.
\newblock Vprop: Variational inference using RMSprop.
\newblock {\em arXiv preprint arXiv:1712.01038}, 2017.

\bibitem{khan2017adaptivenewton}
M.~E.~Khan, W.~Lin, V.~Tangkaratt, Z.~Liu, D.~Nielsen.
\newblock Adaptive-Newton method for explorative learning.
\newblock {\em arXiv preprint arXiv:1711.05560}, 2017.

%
\bibitem{liu2017stein}
Q.~Liu.
\newblock Stein variational gradient descent as gradient flow.
\newblock In {\em Advances in Neural Information Processing systems} (I.~Guyon et al., Eds.), Vol.~30, p.~3118--3126, 2017.

\bibitem{liu2017policy}
Y.~Liu, P.~Ramachandran, Q.~Liu, and J.~Peng.
\newblock Stein variational policy gradient.
\newblock {\em arXiv preprint arXiv:1704.02399}, 2017.

\bibitem{liu2016svgd}
Q.~Liu and D.~Wang.
\newblock Stein variational gradient descent: A general purpose Bayesian
inference algorithm.
\newblock In {\em Advances In Neural Information Processing Systems} (D.~D.~Lee et al., Eds.), Vol.~29, p.~2378--2386, 2016.

\bibitem{liu2017riemannian}
C.~Liu and J.~Zhu.
\newblock Riemannian Stein variational gradient descent for Bayesian inference.
\newblock {\em arXiv preprint arXiv:1711.11216}, 2017.

\bibitem{luenberger1997optimization}
D.~G.~Luenberger.  
\newblock {\em Optimization by vector space methods}.
\newblock John Wiley \& Sons, 1997.

%
\bibitem{stonewton2012}
J. Martin,  L. C. Wilcox, C. Burstedde, and O. Ghattas.
\newblock A stochastic Newton MCMC method for large-scale statistical inverse problems with application to seismic inversion.
\newblock {\em SIAM Journal on Scientific Computing}, 34(3), A1460--A1487, Chapman \& Hall/CRC, 2012

\bibitem{marzouk2016sampling}
Y.~M.~Marzouk, T.~Moselhy, M.~Parno, and A.~Spantini.
\newblock Sampling via measure transport: An introduction.
\newblock {\em Handbook of Uncertainty Quantification}, Springer, p.~1--41, 2016.

%
\bibitem{neal2011hamiltonian}
R.~M. Neal.
\newblock MCMC using Hamiltonian dynamics.
\newblock In {\em Handbook of Markov Chain Monte Carlo} (S.~Brooks et al., Eds.), Chapman \& Hall/CRC, 2011.
%

%
\bibitem{pu2017va}
Y.~Pu, Z.~Gan, R.~Henao, C.~Li, S.~Han, and L.~Carin.
\newblock Stein variational autoencoder.
\newblock {\em arXiv preprint arXiv:1704.05155}, 2017.
%

%
\bibitem{rezende2015variational}
D.~Rezende and S.~Mohamed.
\newblock Variational inference with normalizing flows.
\newblock {\em arXiv:1505.05770}, 2015.
%
\bibitem{spantini2017inference}
A.~Spantini, D.~Bigoni, and Y.~Marzouk.
\newblock Inference via low-dimensional couplings.
\newblock {\em Journal of Machine Learning Research}, to appear.
\newblock {\em arXiv:1703.06131}, 2018.

\bibitem{stuart2010}
A.~M.~Stuart.
\newblock Inverse problems: a Bayesian perspective.
\newblock {\em Acta Numerica}, 19, p. 451--559, 2010.
%
\bibitem{tabak2013family}
E.~G.~Tabak and T.~V.~Turner.
\newblock A family of nonparametric density estimation algorithms.
\newblock {\em Communications on Pure and Applied Mathematics}, p.~145--164, 2013.
%

%
\bibitem{villani2008optimal}
C.~Villani.
\newblock {\em Optimal Transport: Old and New}.
\newblock Springer-Verlag Berlin Heidelberg, 2009.
%
\bibitem{wang2017structured}
D.~Wang, Z.~Zeng, and Q.~Liu.
\newblock Structured Stein variational inference for continuous graphical models.
\newblock {\em arXiv:1711.07168}, 2017.
%

%

%
\bibitem{zhuo2017analyzing}
J.~Zhuo, C.~Liu, N.~Chen, and B.~Zhang.
\newblock Analyzing and improving Stein variational gradient descent for
high-dimensional marginal inference.
\newblock {\em arXiv preprint arXiv:1711.04425}, 2017.

\bibitem{wright1999numerical}
S.~Wright, J.~Nocedal.
\newblock {\em Numerical Optimization}.
\newblock Springer Science, 1999.
  
\end{thebibliography}

\end{document}